\documentclass[10pt]{article}
\usepackage{ijcai20}

\usepackage{times}
\usepackage{soul}
\usepackage{url}
\usepackage[utf8]{inputenc}
\usepackage[small]{caption}
\usepackage{graphicx}
\usepackage{amsmath}
\usepackage{amsthm}
\usepackage{booktabs}
\usepackage{algorithm}
\usepackage{algorithmic}
\usepackage{amssymb}
\usepackage{bm}
\usepackage{dsfont}
\usepackage{xcolor}
\usepackage{xspace}

\usepackage[british]{babel}

\usepackage[colorinlistoftodos, textwidth=40mm, shadow]{todonotes}

\definecolor{todocolor}{rgb}{0.66,0.99,0.99}

\newcommand{\dsR}{\mathds{R}}

\newcommand{\dsone}{\mathds{1}}

\newcommand{\cA}{\mathcal{A}}

\newcommand{\cD}{\mathcal{D}}
\newcommand{\cL}{\mathcal{L}}
\newcommand{\cM}{\mathcal{M}}

\newcommand{\cX}{\mathcal{X}}

\newcommand{\bd}{\bm{d}}

\newcommand{\bW}{\bm{W}}
\newcommand{\bmu}{{\boldsymbol \mu}}
\newcommand{\bnu}{{\boldsymbol \nu}}

\newcommand{\bomega}{{\boldsymbol \omega}}
\newcommand{\bw}{\bomega}
\newcommand{\blambda}{\boldsymbol \lambda}

\newcommand{\sumi}{\sum_{i=1}^K}

\newcommand{\suma}{\sum_{a\in[K]}}

\newtheorem{theorem}{Theorem}

\newtheorem{lemma}[theorem]{Lemma}
\newtheorem*{remark}{Remark}

\newtheorem{proposition}[theorem]{Proposition}

\theoremstyle{definition}

\def\cX{\mathcal{A}}
\newcommand{\E}{\mathbb{E}}
\renewcommand{\P}{\mathbb{P}}
\newcommand{\transpose}{^\mathsf{\scriptscriptstyle T}}
\newcommand{\norm}[1]{\left\lVert#1\right\rVert}

\DeclareMathOperator*{\argmax}{arg\,max}
\DeclareMathOperator*{\argmin}{arg\,min}
\newcommand{\SpectralTaS}{\texttt{\textcolor[rgb]{0.5,0.2,0}{SpectralTaS}}\xspace}

\title{Best Arm Identification in Spectral Bandits}

\author{
Tomáš Kocák\footnote{Contact Author}
\and
Aurélien Garivier
\affiliations
École Normale Supérieur de Lyon\\
\emails
tomas.kocak@gmail.com, aurelien.garivier@ens-lyon.fr
}

\begin{document}

\maketitle

\begin{abstract}
We study best-arm identification with fixed confidence in bandit models with graph smoothness constraint. We provide and analyze an efficient gradient ascent algorithm to compute the sample complexity of this problem as a solution of a \mbox{non-smooth} max-min problem (providing in passing a simplified analysis for the unconstrained case). Building on this algorithm, we propose an asymptotically optimal strategy. We furthermore illustrate by numerical experiments both the strategy's efficiency and the impact of the smoothness constraint on the sample complexity. Best Arm Identification (BAI) is an important challenge in many applications ranging from parameter tuning to clinical trials. It is now very well understood in vanilla bandit models, but real-world problems typically involve some dependency between arms that requires more involved models. Assuming a graph structure on the arms is an elegant practical way to encompass this phenomenon, but this had been done so far only for regret minimization. Add\-ressing BAI with graph constraints involves deli\-cate optimization problems for which the present paper offers a solution.
\end{abstract}

\section{Introduction}

This work is devoted to the optimization of a stochastic function on a structured, discrete domain $\cX = [K] \triangleq \{1,\dots, K\}$. We consider the noisy black-box model: a call to the function at point $a\in\cX$ yields an independent draw of an unknown  distribution $\nu_{a}\in\mathcal{F}$ with mean $\mu_{a}$, where $\mathcal{F}$ is some family of probability laws. At each time step $t\in\mathbb{N}$, a point $A_t\in\cX$ can be chosen (based on past observations) and the corresponding random outcome $Y_t$ of law $\nu_{A_t}$ is observed. 

The signal  $\bmu \triangleq (\mu_1,\,\dots,\,\mu_K)\transpose$ is assumed to be \emph{smooth} in the following sense: $\cX$ is equipped with a fixed and \emph{weighted graph structure} $G$ with adjacency matrix $\bW = (w_{a,b})_{a,b\in\cX}$, and $\bmu$ satisfies the graph smoothness property: 
	\[
	S_G(\bmu) \triangleq  \sum_{a,b\in\cX} w_{a,b} \frac{(\mu_a-\mu_b)^2}{2} = \bmu\transpose\cL\bmu = \norm{\bmu}_{\cL}^2 \leq R
	\]
for some (known) smoothness parameter $R$, where $\cL$ is the \textbf{graph Laplacian} defined as: $\cL_{a,\,b} = -w_{a,\,b}$ for $a\not=b$ and $\cL_{a,\,a} = \sum_{b\not=a}w_{a,\,b}$.
This enforces values of means $(\mu_a)_a$ at two points $a,b\in\cX$ to be close to each another if the weight $w_{a,b}$ is large.

Following a classical framework in statistical testing, a risk parameter $\delta$ is fixed, and the goal is to design an algorithm that successfully maximizes $\mu$ on $\cX$ as fast as possible, with failure probability at most $\delta$. More specifically, the algorithm consists of a \emph{sampling rule} $(\psi_t)_{t\geq 1}$ that chooses to observe $A_t = \psi_t(A_1,Y_1,\dots,A_{t-1}, Y_{t-1})$ at step $t$, and a stopping rule $\tau$ such that whatever the distributions $\bnu = (\nu_a)_{a\in\cX}\in\mathcal{F}^K$, satisfies the constraint \[\P_\nu\Big(A_{\tau+1} \in a^*(\bmu) \triangleq \argmax_{i \in [K]} \mu_i\Big) \geq 1-\delta\;.\] 
Among all such algorithms, the goal is to minimize the sample complexity $\E_\bmu[\tau]$.

This problem is known in the multi-armed bandit literature as \emph{best-arm identification with fixed confidence}~\cite{even2006action,gabillon2012best}. In that context, points of the domain $\cX$ are called \emph{arms} and the set of distributions $\bnu$ is called a model. Bandit models have raised strong interest in the past years, as multi-armed bandits were gaining popularity thanks to their ability to model a variety of real-world scenarios while providing strong theoretical guarantees for the proposed algorithms. For an introduction to the theory of bandit models and a recent list of references, see~\cite{lattimore} and references therein. Best-arm identification, in particular, has many implications in artificial intelligence and data science. 
For example, in  pharmacology, the optimal drug dosage depends on the drug formulation but also on the targeted individuals (genetic characters, age, sex) and environmental interactions. In the industrial context, the shelf life and performance of a manufacturing device depend on its design, but also on its environment and on uncertainties during the manufacturing process. As for targeted recommender systems, their efficiency depends on the ability to understand users’ preferences despite the high variability of choices. Lastly, the performance of a deep neural network depends on its (hyper-)parameters and on the quality of the training set. This short list is of course not comprehensive, but the automation of these representative optimization tasks is one of the great challenges of modern artificial intelligence in the fields of personalized medicine, predictive maintenance, online marketing, and autonomous learning.

The problem of best-arm identification has recently received a precise solution for the vanilla version of the multi-armed bandit problem, where the learner gains information only about the arm played. \cite{garivier2016optimal,GK19PAC} (see also~\cite{Russo16}) have described the information-theoretic barriers of the sample complexity in the form  an instance-optimal lower bound on $\E_\bmu[\tau]$; furthermore, they have proposed a strategy, called Track-and-Stop, that asymptotically reaches this lower bound.

However, the vanilla bandit model is extremely limiting in many applications, and some additional information has to be incorporated into the learning process. Recently, several papers studied multi-armed bandit problems with the ability to capture the rich structure of real-world problems. In particular, \cite{Degenne2019,Degenne2019a} have extended the results mentioned above to classes of possibly structured bandits with possibly different goals. In parallel, a fruitful approach to handle large and organized domains is assume the existence of a graph structure on top of arms which provides additional information. There are two most dominant approaches to use this graph structure. The first approach is to use graphs to encode additional information about  other arms; playing an arm reveals (fully or partially) the rewards of all the neighbors \cite{mannor2011from,alon2013from,kocak2014efficient,alon2014nonstochastic,kocak2016online,kocak2016onlinea}. The second approach is to encode the similarities between arms by the weights of the edges; the rewards of connected arms are similar~\cite{valko2014spectral,kocak2014spectral,kocak2018spectral}.
This second approach, which we adopt in this paper, easily permits to capture a lot of information about the problem by the creation of a \emph{similarity graph}, which explains its wide use in signal processing, manifold, and semi-supervised learning. However, most of the work in multi-armed bandits with structure has been done so far under the regret minimization objective, while best arm identification (a practically crucial objective in autoML, clinical trials, A/B testing, etc.) was mostly neglected. 

\subsection{Our Contributions}
In this paper, we aim to fill this gap by combining the best arm identification work of  \cite{garivier2016optimal} with the spectral setting of \cite{valko2014spectral}, resulting in spectral best arm identification problem. This setting captures a rich variety of real-world problems that have been neglected so far. We also analyze this setting and provide instance-optimal bounds on the sample complexity. Inspired by~\cite{Degenne2019}, we use a game-theoretical point of view on the problem which enables us to obtain two main contributions of this paper:

\paragraph{Contribution 1:} simplified and enlightening game-theoretic analysis for the unconstrained case (case without the spectral constraint) of \cite{garivier2016optimal} that brings new insights to the problem which enable us to extend the analysis to the constrained case.

\paragraph{Contribution 2:} building upon the unconstrained case, we provide and analyze an algorithm that computes the sample complexity as well as the optimal arm pulling proportions in the constrained case. We use this algorithm to propose an asymptotically optimal strategy for fixed-confidence best-arm identification and we provide numerical experiments supporting the theoretical results and showing the interest of the graph constraint for reducing the sample complexity. 

\subsection{Assumptions}
A standard assumption on the family $\mathcal{F}$ of considered probability laws is that they form an exponential family of distributions parametrized  by theirs means so that a bandit problem is fully described by the vector $\bmu \triangleq (\mu_i,\,\dots,\,\mu_K)\transpose$. To avoid technicalities and to emphasize the methodological contributions of this paper (ie. how to perform optimal best-arm identification with graph regularity constraints), we focus here on the case of Gaussian distributions with variance $1$. 
For the same reason, we assume here that the vector $\bmu$ has a unique maximum also denoted by $a^*(\bmu)$.
We furthermore denote by $\displaystyle{
\cM_R = \{ \blambda \in \dsR^K: \blambda\transpose\cL\blambda \le R \}}$ the constrained set of considered signals (bandit problems), and by $\mu_* \triangleq \max_{i \in [K]} \mu_i = \mu_{a^*(\bmu)}
$ the maximum of the signal $\bmu$ (we identify the best arm with s singleton $a^*(\bmu)$). 

\subsection{Paper Structure}
The paper is organized as follows: Section~\ref{sec:samplecomplexity} discusses the information-theoretic lower bound and a simple algorithm for the unconstrained case, previously analyzed by \cite{garivier2016optimal}, with focus on new techniques that considerably simplify the analysis. The constrained case requires more care: we propose in Section~\ref{sec:mirror} a mirror gradient ascent algorithm to compute optimal arm allocation and show non-asymptotic bounds of convergence for this algorithm. Using this procedure permits us in Section~\ref{sec:algo} to present the \SpectralTaS algorithm as an extension of the Track-and-Stop devoted to best-arm identification with graph regularity constraint, that reaches the instance-optimal sample complexity. We illustrate its performance and the impact of the smoothness constraint on the sample complexity by reporting some numerical experiments.

\section{Sample Complexity}\label{sec:lowerbound}\label{sec:samplecomplexity}
The general information-theoretical lower bound analysis provided by \cite{garivier2016} applies to any set of bandit problems, in particular to the set of all the problems with smoothness bounded from above by $R$.
\begin{proposition}
\label{prop:lowerbound}
For any $\delta$-correct strategy and any bandit problem $\bmu$
\[
  \E_\bmu[T_\delta] \ge T_R^*(\bmu) k(\delta,\,1-\delta)
\]
\noindent where
\begin{equation}
  T_R^*(\bmu)^{-1} \triangleq \sup_{\bomega\in\Delta_K}\inf_{\blambda\in\cA_R(\bmu)} \suma\omega_a k(\mu_a,\lambda_a) \label{eq:maximin}
\end{equation}
for $\Delta_K$ being $K$-dimensional simplex, $k(\mu_a,\lambda_a)$ KL-divergence between distributions parametrized by $\mu_a$ and $\lambda_a$, and for the set $\cA_{R}(\bmu)$ of bandit problems with different best arm than $\bmu$, defined as
\begin{align*}
\cA_{R,i}(\bmu) &\triangleq \left\{ \blambda \in \cM_R:  \lambda_i \ge\lambda_{a^*(\bmu)} \right\} \\
\cA_{R}(\bmu) &\triangleq \cup_{i\not= a^*(\bmu)} \cA_{R,i}(\bmu).
\end{align*} 
\end{proposition}
Indeed, even if the constrained case is not explicitly covered there, this result is easily shown by following the lines of the proof of Theorem 1 in \cite{garivier2016},  with $\cA_R(\bmu)$ as the set of alternatives to $\bmu$. 

This lower bound proves to play an important role in designing algorithms for the best arm identification since, roughly speaking, the time needed to distinguish between bandit problems $\bmu$ and $\blambda$, if the number of pulls of arm $i$ is proportional to $\omega_i$, scales with inverse of $\suma\omega_a k(\mu_a,\lambda_a)$. Therefore, the minimization part of problem (\ref{eq:maximin}) selects the alternative problem which is the most difficult to distinguish from $\bmu$ while playing according to $\bomega$ while, on the other hand, the maximization part of problem (\ref{eq:maximin}) chooses $\bomega$ such that the expected stopping time is as small as possible, even in the worst case scenario chosen by the minimization part of the problem.

Thus, a sampling strategy playing according to $\bomega^*(\bmu)$ that maximizes expression (\ref{eq:maximin}) is optimal in the worst case, and having a procedure that computes optimal $\bomega^*(\bmu)$ enables to design a best arm identification algorithm with optimal sample complexity. 

\medskip

In this section, we provide an algorithm to compute this sample complexity $T_R^*(\bmu)$ as well as the optimal proportion $\bomega^*(\bmu)$ of the arm allocation for the learner. We proceed in three steps.

\paragraph{Step 1:}we introduce the \textit{best response oracle} that computes the best response $\blambda^*(\bomega)$ to fixed $\bomega$ by solving the minimization part of problem (\ref{eq:maximin}) for chosen $\bomega$:
\[
  \blambda^*(\bomega) \in \argmin_{\blambda\in\cA_R(\bmu)} \suma \omega_ak(\mu_a,\lambda_a)\;.
\]
\paragraph{Step 2:} we show that substituting the minimization part of problem (\ref{eq:maximin}) by $\blambda^*(\bomega)$, the resulting maximization problem is concave with respect to $\bomega$. Moreover, supergradient for this problem is computable using the same best response oracle we used for the minimization part of the problem.
\paragraph{Step 3:} using supergradient we can apply any supergradient ascent algorithm to find optimal arm allocation $\bomega^*$. The algorithm of our choice is mirror ascent with generalized negative entropy as a mirror map. We chose this algorithm due to its strong guarantees on convergence rate in simplex setup however, even basic gradient-based algorithms could be applied. The choice of mirror ascent algorithm will be apparent later in Section \ref{sec:mirror}, we show that this algorithm enjoys convergence rate proportional to $\sqrt{\log K}$ instead of $\sqrt{K}$ of the basic supergradient ascent algorithm.

\medskip

As briefly mentioned in the introduction, for the sake of simplicity of presentation, we assume that the distributions $\nu_{\mu_a}$ are Gaussian. This simplifies some of the calculations while our bounds remain true even in the case of sub-Gaussian random variables, including a wide class of distributions among which any distributions with bounded support or Bernoulli distribution.

Even though using a general exponential family of probability distributions uses similar techniques, all the proofs are more technical and rather out of the scope of a conference paper.

With assumption that the distributions associated with arms are Gaussian with normalized variance, KL-divergence of $\nu_{\mu_a}$ and $\nu_{\lambda_a}$ can be expressed as
\[
k(\mu_a,\lambda_a) = \frac{(\mu_a-\lambda_a)^2}{2}\;.
\]

\subsection{Best Arm Identification without Spectral Constraint}
This section is dedicated to the vanilla setting of \cite{garivier2016optimal}. Even though this problem can be seen as a special case of spectral setting either (for the edgeless graph or for smoothness parameter $R$ being $\infty$), we have decided to analyze it separately. The reason is that we can demonstrate techniques later used in the constrained case while proving the main result of \cite{garivier2016optimal} with more insightful arguments that are significantly more elegant. Also, not having the spectral constraint allows us to have best response oracle with a closed-form solution which implies a simple way of finding optimal arm allocation $\bomega^*$. This is very different from the spectral setting where best response oracle does not have a closed-form solution and therefore, we can not avoid a numerical procedure to find $\bomega^*$.

\begin{theorem}\label{thm:nonspec}
Without loss of generality, assume that $\mu_1 > \mu_2 \ge\dots\ge\mu_K$ and define $\{x_a(c)\}_{a=2}^K$, for some positive constant $c$,  recursively as
\begin{align*}
x_K(c) &= c\\
x_{a-1}(c) &= (1+x_a(c))\left(\frac{\mu_1-\mu_{a-1}}{\mu_1-\mu_a}\right)^2 - 1
\end{align*}
for $2<a\le K$. Let $f(c)$ be a function with parameter $c$ defined as
\[
f(c) =\sum_{a=2}^Kx_a(c)^2
\]
Then there exist $c^*\in\dsR^+$ s.t. $f(c^*) = 1$ and we obtain $\bomega^*(\bmu)$ as 
\begin{align*}
\omega^*_1(\bmu) &= \frac{1}{1+\sum_{a=2}^Kx_a(c^*)} \\
\omega^*_a(\bmu) & = x_a(c^*)\omega^*_1(\bmu)
\end{align*}
\end{theorem}

\begin{remark}
Since $\mu_1 > \mu_{a-1} \ge \mu_{a}$, we see that
\[
\left(\frac{\mu_1-\mu_{a-1}}{\mu_1-\mu_a}\right)^2 \le 1\;,
\]
which in consequence means that. If $x_{a}(c)$ is negative, $x_{a-1}(c)$ is negative as well. If $x_{a}(c)$ is positive, $x_{a-1}(c) \le x_{a}(c)$. Thus, $f(0) \le 0$. Moreover, $f$ is an increasing and continuous function with $\lim_{c\to\infty}f(c) = \infty$. Therefore, the existence of $c^*$ is guaranteed.  
\end{remark}

In the rest of this section, we build the necessary tools to provide the proof of Theorem \ref{thm:nonspec}.

\subsubsection{Best Response Oracle - Vanilla Setting}
\label{sec:vanillaoracle}
In the case of smoothness $R = \infty$, the spectral constraint is always satisfied which provides a very simple closed form solution for the best response oracle given by the following lemma.
\begin{lemma}\label{lem:vanillaoracle}
Let $\bomega$ be a fixed vector from $\Delta_K$. Then the best response $\blambda^i(\bomega)\in\cA_{\infty,i}(\bmu)$ with the best arm $i$ is
\[
\blambda^i(\bomega) = (t,\mu_2,\dots,\mu_{i-1},t,\mu_{i+1},\dots,\mu_K)\transpose
\]
for $t$ being a weighted average of $\mu_1$ and $\mu_i$ with weights $\omega_1$ and $\omega_i$
\[
t = \frac{\mu_1\omega_1 + \mu_i\omega_i}{\omega_1+\omega_i}\;.
\]
\end{lemma}
\noindent The proof of this Lemma can be obtained by simple calculus.

\subsubsection{Game-Theoretical Point of View}
Our optimization problem (\ref{eq:maximin}) can be seen as a zero sum game where Player 1 plays a vector from simplex $\Delta_K$ and Player 2 plays point $\blambda$ from $\cA_R(\bmu) $. We would like to have guarantees on the existence of Nash equilibrium however $\cA_R(\bmu)$ is not a convex set which poses a problem. Also, we can not directly convexify the problem since it would change the value of $T_R^*(\bmu)^{-1}$. To get around this obstacle, we define
\begin{align*}
\cD_{R,i}(\bmu) &\triangleq \left\{ (k(\mu_1,\lambda_1),\dots,\,k(\mu_K,\lambda_K))\transpose:\blambda \in\cA_{R,i}(\bmu)\right\} \;,\\
\cD_R(\bmu) &\triangleq \cup_{i\not=a^*(\bmu)} \cD_{R,i}(\bmu)\;.
\end{align*}
This enables us to rewrite the optimization problem in the following form
\begin{equation}
  T_R^*(\bmu)^{-1} \triangleq \sup_{\bomega\in\Delta_K}\inf_{\bd\in\cD_R(\mu)} \bomega\transpose\bd  \label{eq:game}
\end{equation}
where Player 2, instead of playing $\blambda$, plays vectors of divergences $(k(\mu_1,\lambda_1),\,\dots,\,k(\mu_K,\lambda_K))\transpose$ where the element on position $i$ of the vector is divergence $k(\mu_i,\lambda_i)$ of the distributions corresponding to $i$-th elements of $\bmu$ and $\blambda$. This still haven't solve our issue with non-convexity of set  $\cD_R(\mu)$, however, now we can simply optimize over the convex hull of $\cD_R(\mu)$
\begin{equation}
  T_R^*(\bmu)^{-1} \triangleq \sup_{\bomega\in\Delta_K}\inf_{\bd\in Conv(\cD_R(\mu))} \bomega\transpose\bd  \label{eq:convgame}
\end{equation}
and the following lemma guarantees that the values of optimization problems (\ref{eq:game}) and (\ref{eq:convgame}) are the same.

\begin{lemma}\label{lem:conv}
Let $\bomega$ be a vector in $\dsR^K$ and $\cD$ be a compact subset of $\dsR^K$ then
\[
\inf_{\bd\in \cD} \bomega\transpose \bd = \inf_{\bd\in Conv(\cD)} \bomega\transpose \bd
\]
where $Conv(\cD)$ is the convex hull of $\cD$.
\end{lemma}
Now we are left a zero sum game (\ref{eq:convgame}) with convex and compact sets of actions for both players. This guarantees existence of Nash equilibrium with the best actions denoted by $\bomega^*(\bmu)$ and $\bd^*(\bmu)$. Whole purpose of seeing (\ref{eq:convgame}) as a zero-sum game is the following lemma
\begin{lemma}\label{lem:ones}
Let $\bomega^*(\bmu)$ and $\bd^*(\bmu)$ are vectors for which Nash equilibrium of game (\ref{eq:convgame}) is attained then there exist a constant $c$ such that
\[
d^*_i(\bmu) = c, \qquad \text{for all } i\in[K].
\]  
\end{lemma}

\subsubsection{Proof of Theorem \ref{thm:nonspec}}
Now we have all necessary tools to prove Theorem \ref{thm:nonspec}. Note that we still assume that $\mu_1 > \mu_2 \ge\dots\ge\mu_K$. Let $\bomega^*$ be the Nash equilibrium strategy. As we showed in Lemma \ref{lem:vanillaoracle}, the best response $\blambda^i( \bomega^*)\in\cA_{\infty,i}(\bmu)$ to $\bomega^*$ has a particular form
\[
\blambda^i(\bomega^*) = (t,\mu_2,\dots,\mu_{i-1},t,\mu_{i+1},\dots,\mu_K)\transpose\;,
\]
where $\displaystyle{t = \frac{\mu_1\omega^*_1 + \mu_i\omega^*_i}{\omega^*_1+\omega^*_i}
}$.
This $\blambda^i(\bomega^*)$ corresponds to point $\bd^i(\bomega^*)$ from $\cD_{R,i}$  defined as
\[
\bd^i(\bomega^*) = (y_i,0,\dots,0,z_i,0,\dots,0)\transpose
\]
for $y_i = (\mu_1-u_i)^2/2$ and $z_i = (\mu_i-u_i)^2/2$. It is important to notice that the only two non-zero elements are at positions $1$ and $i$. The Nash equilibrium strategy $\bd^*$ of Player 2 can be expressed as a convex combination of optimal points from $\cD_R(\bmu)$ as showed in Lemma \ref{lem:conv}. The only candidates are points $\bd^i(\bomega^*)$ for $i\in\{2,\dots,K\}$ since all the other points from $\cD_R(\bmu)$ are sub-optimal. Moreover, using Lemma \ref{lem:ones} we know that all the elements of $\bd^*$ are equal which in particular means that $\bd^*$ needs to be a convex combination of all $K-1$ vectors $\bd^i(\bomega^*)$ and all of them are equally good and from the definition of $\bd^i(\bomega^*)$ we have
\[
 (\bomega^*)\transpose\bd^i(\bomega^*) = (\bomega^*)\transpose\bd^j(\bomega^*)
\]
for every $i,j\ge2$. This expression can be further modified to obtain
\[
\frac{\omega^*_i}{\omega^*_1} = \left(1+\frac{\omega^*_j}{\omega^*_1}\right)\left(\frac{\mu_1-\mu_{i}}{\mu_1-\mu_j}\right)^2 - 1
\]
Denoting $\frac{\omega^*_2}{\omega^*_1}$ by $c^*$ and using $i = j-1$ we directly obtain that the recurrence from the theorem holds for $x_i(c^*) =  \frac{\omega^*_i}{\omega^*_1}$. Therefore, it is enough to know $c^*$ and $\omega^*_1$ to reconstruct $\bomega^*$.

Lemma \ref{lem:ones} provides condition for $\bd^*$. We know that it can be written as a convex combination of $K-1$ vectors $\bd^i(\bomega^*)$. This is possible only if vector $\dsone$ of ones can be expressed as their linear combination which gives us condition
\[
\sum_{i=2}^K \frac{\bd^i(\bomega^*)}{z_i} = \dsone \;.
\] 
We know that all the elements of the resulting vector are equal to 1 except for the first element which has a value
\begin{align*}
\sum_{i=2}^K \frac{y_i}{z_i} = \sum_{i=2}^K\left(\frac{\omega^*_i}{\omega^*_1}\right)^2 = \sum_{i=2}^K x_i(c^*)^2 = f(c^*).
\end{align*}
Therefore, this value has to be 1 too, which gives us the second part of the theorem.
Now we can recover all the ratios $\omega^*_i/\omega^*_1$ and we also know that the sum of all $\omega_i$ is 1 which provides all the necessary ingredients for the proof of the last part of the theorem.

\subsection{Best Response Oracle - Spectral Setting}
\label{sec:spectraloracle}
As we showed in Section \ref{sec:vanillaoracle}, Lemma \ref{lem:vanillaoracle}, finding the best response to $\bomega$, in the vanilla setting, is not a difficult problem since the optimization problem involves only linear constraints. The situation in the spectral setting is a little bit more complicated.

In this part, we again assume that $\mu_1 > \mu_2\dots\ge \mu_K$ and focus on the best response oracle
\[
\blambda^*(\bomega) = \argmin_{\blambda\in\cA_R(\bmu)}{\suma \omega_a\frac{(\mu_a-\lambda_a)^2}{2}}
\]
 in spectral setting, where $R$ is a finite upper bound on smoothness of $\bmu$ with respect to graph $G$ given by its Laplacian $\cL$. We first find best responses
\[
\blambda^i(\bomega) = \argmin_{\blambda\in\cA_{R,i}(\bmu)}{\suma \omega_a\frac{(\mu_a-\lambda_a)^2}{2}}
\]
with respect to convex sets $\cA_{R,i}(\bmu)$ and then take the best out of these $K-1$ vectors.

First of all, notice that if the response of vanilla oracle denoted by $\blambda_\infty^i(\bomega)$ has smoothness smaller than $R$, the response of spectral oracle should be the same. In the case where $S_G(\blambda_\infty^i(\bomega)) \ge R$ we can restrict our search and look for $\blambda$ with $S_G(\blambda)$ exactly $R$ thanks to the following lemma.

\begin{lemma}\label{lem:saturation}
Let $\blambda_\infty^i(\bomega)$ be the response of vanilla oracle such that $\blambda_\infty^i(\bomega)\transpose\cL\blambda_\infty^i(\bomega) > R$ then the response of spectral oracle $\blambda^i(\bomega)$ satisfies $\blambda^i(\bomega)\transpose\cL\blambda^i(\bomega) = R$.
\end{lemma}

\subsection{Best Response Oracle Implementation}
Now lets consider the case where vanilla oracle produces a vector that does not satisfy smoothness constraint. In this case we know, using Lemma \ref{lem:saturation}, that the best response $\blambda^i(\bomega)$ of spectral oracle has smoothness $R$. Therefore, we can use standard Lagrange multiplier method to solve this problem.
\[
F(\blambda,\gamma) \triangleq \suma\omega_a\frac{(\mu_a-\lambda_a)^2}{2} + \gamma\left(\blambda\transpose\cL\blambda - R\right)
\]
We should not forget that we still need to ensure that $\lambda_1 = \lambda_i$ since we want to find the solution in $\cA_{R,i}(\bmu)$ with best arm $i$. Therefore, we can simplify the notation and later calculations by the following definitions
\begin{itemize}
\item $\tilde\blambda$: created from $\blambda$ by removing the first component
\item $\tilde\bomega$: created from $\bomega$ by removing the first component
\item $\tilde\bmu$: created from $\bmu$ by averaging components 1st and $i$-th components w.r. to $\bomega$ and removing the first component
\[
\tilde\mu_{i-1} = \frac{\omega_1\mu_1+\omega_i\mu_i}{\omega_1+\omega_i} \;.
\]
\item $\tilde\cL$: created from $\cL$ by adding the first row and column to $i$-th row and column, adjusting diagonal element $\cL_{i,i}$ so that the sum in the $i$-th row is zero, removing the first row and column 
\end{itemize}
Using previous definitions we can define 
\[
\tilde F(\tilde\blambda,\gamma) \triangleq \sum_{a\in[K-1]}\tilde\omega_{a}\frac{(\tilde\mu_a-\tilde\lambda_a)^2}{2} + \gamma(\tilde\blambda\transpose\tilde\cL\tilde\blambda-R)
\]
for which $\tilde F(\tilde\blambda,\gamma) = F(\blambda,\gamma) + c$ holds, assuming that $\lambda_1 = \lambda_i$. The identity can be checked by using previous definitions. The main advantage of this transformation is that now we don't have any constraint on $\tilde\blambda$. Now we can take partial derivatives with respect to $\tilde\lambda_j$ and $\gamma$ and set them equal to 0 in order to find the best $\tilde\blambda$ and later reconstruct $\blambda$:
\newcommand{\bOmega}{{\boldsymbol \Omega}}
\begin{align*}
\nabla_{\tilde\blambda}\tilde F(\tilde\blambda,\gamma) &= \tilde\bOmega (\tilde\blambda - \tilde\bmu) + 2\gamma\tilde\cL\tilde\blambda = \boldsymbol{0} \;,\\
\frac{\partial F(\tilde\blambda,\gamma)}{\partial\gamma} &= \tilde\blambda\transpose\tilde\cL\tilde\blambda - R = 0\;,
\end{align*}
where $\tilde\bOmega$ is a diagonal matrix with $\tilde\bomega$ on its diagonal. First expression is just a system of $K-1$ linear equations with parameter $\gamma$ with solution
\[
\tilde\blambda(\gamma) = (\tilde\bOmega+2\gamma\tilde\cL)^{-1}\tilde\bOmega\tilde\bmu \;.
\]
To find the best $\gamma$ so that also the second expression is true, we can use for example bisection method.

The last step is choosing $i>1$ that minimizes $F(\blambda,\gamma)$.

\subsubsection{Properties of Maximization Problem}
Before we proceed to the algorithm computing optimal $\bomega^*(\bmu)$ we need the following two lemmas to be able to compute supergradients and guarantee the convergence of the algorithm.

\begin{lemma}
\label{lem:linearinf}
Let $\cD\subseteq \dsR^K$ be a compact set. Then function $f: \Delta_K \to \dsR$ defined as
$
f(\bomega) = \inf_{\bd\in\cD} \bomega\transpose\bd
$ 
Is a concave function and
$
\bd^*(\bomega) = \argmin_{\bd\in\cD} \bomega\transpose\bd
$
is a supergradient of $f$ at $\bomega$.
\end{lemma}

\begin{lemma}\label{lem:lipsch}
Let $f:\Delta_K \to \dsR$ be a function such that
\[
 f(\bomega) = \inf_{\blambda\in\cA_R(\bmu)}\sumi \omega_ik(\mu_i,\lambda_i)\;.
\]
Then function $f$ is $L$-Lipschitz with respect to $\norm{\,\cdot\,}_1$ for any
\[
L\ge\max_{i,j\in[K]}k(\mu_i,\mu_j)\;.
\] 
\end{lemma}
\section{Best Arm Allocation Algorithm}\label{sec:mirror}
%
Now that we are able to compute best response $\blambda^*(\bomega)$ and therefore, supergradient of concave function $f(\bomega)=\suma\omega_a k(\mu_a,\lambda^*_a(\bomega))$ at $\bomega$, we have all the ingredients needed for any supergradient-based algorithm to compute best arm allocation $\bomega^* = \argmax_{\bomega\in\Delta_K}f(\bomega)$. The algorithm of our choice is mirror ascent algorithm with generalized negative entropy as the mirror map:
\[
\Phi(\bomega) = \suma(\omega_i\log(\omega_i)-\omega_i)\;.
\]

\begin{theorem}
Let $\bomega_1 = (\frac{1}{K},\dots,\frac{1}{K})\transpose$ and learning rate $\eta = \frac{1}{L}\sqrt{\frac{2\log K}{t}}$. Then mirror ascent algorithm optimizing function $f$ defined on $\Delta_K$ with generalized negative entropy $\Phi$ as the mirror map enjoys the following guarantees
\[
f(\bomega^*) - f\left(\frac{1}{t}\sum_{s=1}^t\bomega_s\right) \le L\sqrt{\frac{2\log K}{t}}
\]
for any $L \ge \max_{i,j\in[K]}k(\mu_i,\mu_j)$.
\end{theorem}
\begin{proof}
It is easy to show that $\Phi$ is $1$-strongly convex w.r.t. $\norm{\,\cdot\,}_1$ and $\sup_{\bomega\in\Delta_K}\Phi(\bomega) - \Phi(\bomega_1) \le \log K$ since 
\[
\suma \left[\omega_a\log \omega_a - \frac{1}{K} \log\frac{1}{K}\right] \le  \suma -\frac{1}{K} \log\frac{1}{K} = \log K
\]
holds for any $\bomega\in\Delta_K$. Using Lemma~\ref{lem:lipsch} we have that $f$ is $L$-Lipschitz w.r.t. $\norm{\,\cdot\,}_1$ for any $L \ge \max_{i,j\in[K]}k(\mu_i,\mu_j)$. This gives us all the necessary ingredients mirror ascent guarantees in Theorem 4.2 from \cite{bubeck2015}.
\end{proof}

\subsection{Value of Spectral Constraint}\label{sec:theoreticaltime}
Having an assumption that captures the problem might improve the learning speed significantly. We demonstrate this effect in a simplistic scenario where $\bmu = (0.9, 0.5, 0.6)\transpose $, graph has only one edge between nodes $2$ and $3$, and the range of $R$ is from $0.01 = \bmu\transpose\cL\bmu$ to $0.1$. $R=0.1$ is completely non-restrictive and the best response to every $\bomega$ is the same as in the best arm identification problem without the spectral constraint. This can be seen on Figure \ref{fig:frk3} where we plot the value of $T_R^*(\bmu)$ (red curve) as a function of $R$ which is proportional to the stopping time lower bound.

\begin{figure}[t]
\includegraphics[width = .49\textwidth]{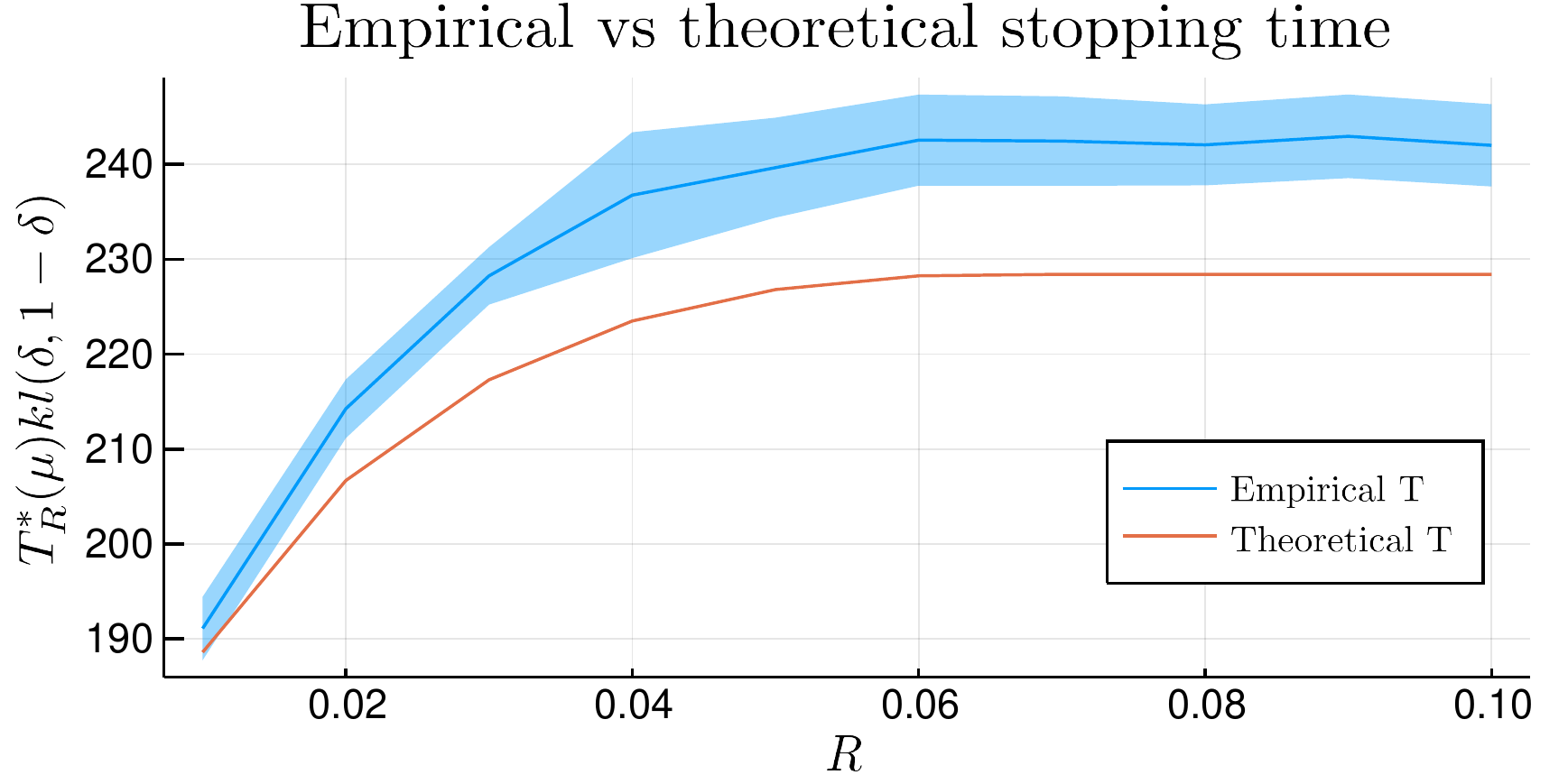}
\vskip -0.7em
\caption{Theoretical vs empirical expected stopping time as a function of smoothness parameter $R$ with optimal value of $R$ being $0.01$.}
\label{fig:frk3}
\vskip -0.5em
\end{figure}

\section{Spectral BAI Algorithm}\label{sec:algo}
The algorithm for the best arm identification with the spectral constraint \SpectralTaS (Algorithm \ref{alg:SpectralTaS}) is a variant of Track-and-Stop algorithm introduced in \cite{garivier2016}. We discuss the main ingredients of the algorithm  in the next part of this section.

\paragraph{Sampling rule.} As a by-product of the lower bound analysis, Proposition~\ref{prop:lowerbound} provides the existence of \emph{optimal sampling weights} $\bomega^*(\bmu)$ that need to be respected in order to reach the optimal sample complexity. \SpectralTaS simply tracks, at every timestep, some guess of these optimal proportions that is obtained by solving the sample complexity optimization problem associated with the \emph{current estimates} $\hat\bmu_t$ of the means $\bmu$. In order to capture a possibility of the initial underestimation of an arm, some level of exploration is needed and enforced by the algorithm: for every $\varepsilon\in(0, 1/K]$, let $\bomega^{*,\varepsilon}(\bmu)$ be an $L^\infty$ projection of $\bomega^*(\bmu)$ onto $\Delta_K^\varepsilon$ defined as $\big\{(w_1,\dots,w_K)\in[\varepsilon, 1]^K : w_1+\dots+w_K=1\big\}$. Then the sampling rule is
\begin{equation}\label{eq:samplerule}
A_{t+1} \in \argmax_{a\in[K]} \sum_{s=0}^t \omega_a^{*,\varepsilon_s}\big(\hat{\bmu}(s)\big) - N_a(t)\;.
\end{equation}
As shown in~\cite{garivier2016optimal} (Lemma 7 and Proposition 9), the choice $\varepsilon_s = (K^2 + s)^{-1/2}/2$ ensures that the number  $N_a(t)$ of draws of arm $a$ converges almost-surely to $\bw_a^*(\mu)$ as $t$ goes to infinity.
\begin{algorithm}[t]
\caption{\SpectralTaS}
\label{alg:SpectralTaS}
\begin{algorithmic}[1]
\STATE \textbf{Input and initialization:}
\STATE \quad $\cL:$ graph Laplacian
\STATE \quad $\delta:$ confidence parameter
\STATE \quad $R:$ upper bound on the smoothness of $\bmu$
\STATE \quad Play each arm $a$ once and observe rewards $r_a$
\STATE \quad $\hat\bmu_1 = (r_1,\dots,r_K)\transpose:$ empirical estimate of $\bmu$
\WHILE {Stopping Rule \eqref{eq:stoprule} not satisfied}
\STATE Compute $\bomega^*(\hat\bmu_t)$ by mirror ascent
\STATE Choose $A_{t}$ according to Sampling Rule \eqref{eq:samplerule}
\STATE Obtain reward $r_t$ of arm $A_t$
\STATE Update $\hat\bmu_t$ according to $r_t$
\ENDWHILE
\STATE \quad Output arm $A_* = \argmax_{a\in[K]} \hat{\mu}_a$
\end{algorithmic}
\end{algorithm}
\paragraph{Stopping rule.}
The algorithm should stop as soon as it has gathered sufficient evidence on the superiority of one of the arms with risk $\delta$. The design of an optimal sequential testing procedure for the hypothesis $\mu_a>\max_{b\neq a} \mu_b$ can be traced back to~\cite{Chernoff59}, and is discussed in detail in~\cite{GK19PAC}. We simply recall its form here: for two arms $a,b\in[K]$, denote by $\hat{\mu}_{a,b}(t)= (N_a(t)\hat{\mu}_a(t) + N_b(t)\hat{\mu}_b(t))/(N_a(t) + N_b(t))$  and by $Z_{a,b} = \mathrm{sign}\big(\hat{\mu}_a(t) - \hat{\mu}_b(t)\big) \big( N_a(t) (\hat{\mu}_a(t)-\hat{\mu}_{a,b}(t))^2 + N_b(t) (\hat{\mu}_b(t)-\hat{\mu}_{a,b}(t))^2\big)/2 $ the generalized likelihood ratio statistics for the test $\mu_a>\mu_b$. Then the stopping rule is given by
\begin{equation}\label{eq:stoprule}
\tau = \inf \Big\{t\in\mathbb{N} : \max_{a\in[K]} \min_{b\neq a} Z_{a,b}(t) > \beta(t,\delta) \Big\}\;,
\end{equation}
where $\beta(\cdot,\cdot)$ is a threshold function to be chosen typically slightly larger than  $\log(1/\delta)$. Theorem 10 in~\cite{garivier2016optimal} shows that the choice $\beta(t,\delta) = \log(2t(K-1)/\delta)$ and $A_{\tau+1} = \argmax_{a\in[K]} \hat{\mu}_a(\tau)$ yields a probability of failure $\P_\nu\left(A_{\tau+1} \notin a^*(\bmu)\right) \leq \delta$.

\paragraph{Empirical evaluation.}
Using the same experiment setting as in Section \ref{sec:theoreticaltime} we have evaluated \SpectralTaS for 10 different values of $R$ ranging from true $R = \bmu\transpose\cL\bmu = 0.01$ to $R= 0.1$ and we plotted theoretical stopping time (red line) as well as average and standard deviation of 20 runs of \SpectralTaS algorithm (blue line). Figure \ref{fig:frk3} shows that the algorithm can utilize the spectral constraint and improve stopping time significantly whenever the value of $R$ is close to the real smoothness of the problem, as suggested by theory.

\section* {Acknowledgements}
Aurélien Garivier (ANR chaire SeqALO)  and Tomáš Kocák acknowledge the support of the Project IDEXLYON of the University of Lyon, in the framework of the Programme Investissements d'Avenir (ANR-16-IDEX-0005).

\newpage

\bibliographystyle{named}
\bibliography{library}

\end{document}